%% file: arxiv.tex
\documentclass[10pt,twocolumn,letterpaper]{article}

\usepackage{cvpr}              
\input{preamble}
\definecolor{cvprblue}{rgb}{0.21,0.49,0.74}
\usepackage[pagebackref,breaklinks,colorlinks,allcolors=cvprblue]{hyperref}
\usepackage{mathtools}
\usepackage{amsmath}
\newtheorem{proposition}[theorem]{Proposition}
\def\confName{CVPR}
\def\confYear{2026}

\title{
Resolving Endpoint Underfitting in Diffusion Bridges via Noise Alignment
}

\author{
{Yurong Gao$^{1\dagger}$} \quad
{Zicheng Zhang$^{2\dagger}$} \quad
{Congying Han$^{1}$\thanks{Corresponding author}} \quad
{Tiande Guo$^{1}$} \quad
{Xinmin Qiu$^{1}$} \\
\normalsize \textsuperscript{1}University of Chinese Academy of Sciences \quad
\textsuperscript{2}JD.com \\
\small $^{\dagger}$Equal contribution.
}

\begin{document}
\maketitle
\input{sec/0_abstract}    
\input{sec/1_intro}

\input{sec/relatedwork}

\input{sec/preliminaries}

\input{sec/method}

\input{sec/experiment}

\input{sec/Conclusion}

{
    \small
    \bibliographystyle{ieeenat_fullname}
    \bibliography{arxiv}
}
\input{sec/X_suppl}


\end{document}

%% file: sec/0_abstract.tex
\begin{abstract}
Diffusion bridge models offer a powerful framework for connecting two data distributions, such as in image restoration and translation. Many existing methods learn this bridge by mimicking the score-matching formulation of standard diffusion models. In this work, we find that this way leads to an anomalous underfitting phenomenon near the target endpoint,  as the process approaches the target distribution ($t \to 0$). 
This underfitting, characterized by significant drift in the predicted variance and direction, results from an excessively large discrepancy in noise levels between the network's input and its regression target.
To resolve this issue, we propose the 
Noise-Aligned Diffusion Bridge (NADB). 
Our approach reformulates the diffusion bridge by first employing a mean network to provide a cleaner conditional target, and then introducing a novel, noise-aligned mapping relationship. This new formulation resolves the noise mismatch and corrects the underfitting near the target endpoint. 
Experimental validation across multiple image restoration and image translation tasks demonstrates the effectiveness of our approach. Code is available at \url{https://github.com/gyr02/NADB}.

\end{abstract}

%% file: sec/1_intro.tex
\section{Introduction}
\label{sec:intro}
\begin{figure}[t!]
    \centering
        \includegraphics[width=0.47\textwidth]{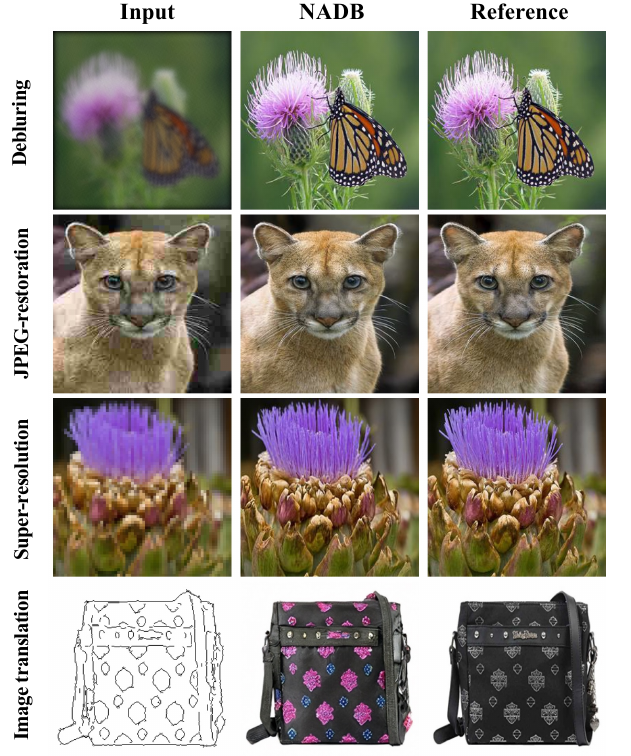}
        \caption{Sampled results from our proposed Noise-Aligned Diffusion Bridge (NADB) for restoring various
degraded images.}
        \label{fig:zhutu}
\end{figure}
Image restoration~\cite{rudin1992nonlinear,freeman2002example,buades2005non,dabov2007image,krishnan2009fast} has recently seen a paradigm shift dominated by generative models~\cite{ledig2017photo,wang2018esrgan,ho2020denoising,kawar2022denoising,wang2022zero}. Among these, diffusion models~\cite{ho2020denoising,kawar2022denoising,wang2022zero} have emerged as the state-of-the-art framework, typically operating by learning a score function to reverse a stochastic noising process. In the context of restoration, this reverse process is conditionally guided by the low-quality input image~\cite{kawar2022denoising,wang2022zero}. However, this standard formulation is arguably circuitous, as it often requires mapping a low-quality input all the way to pure noise before reconstructing the target. A fundamentally different and more natural perspective is offered by diffusion bridge models~\cite{de2021diffusion,chen2022likelihood,liu20232,Wang2025ResidualDB}. Rather than treating distributions in isolation, diffusion bridges learn a direct stochastic trajectory from the low-quality distribution to the high-quality one. This approach avoids the destructive-reconstructive cycle and offers a more principled formulation for mapping between the degraded and clean image manifolds, aligning perfectly with the core objective of restoration.

Within this domain, I2SB~\cite{liu20232} is a pioneering method that operationalizes the bridge concept by mimicking the standard diffusion framework. It learns the score function of the bridge process, enabling the network to predict the target distribution, and its sampler iteratively refines the low-quality image, akin to a standard denoising process. This mimicry, however, raises a critical question: is the standard diffusion training scheme optimal for diffusion bridge models? Our analysis of the I2SB network's output across the entire time horizon reveals a significant flaw: an anomalous underfitting phenomenon that emerges near the $t \to 0$ endpoint (referred to as target endpoint ), as evidenced in Fig.~\ref{fig:bluruni}. 
This underfitting manifests as a severe distortion in the predicted variance and erroneous directional predictions. This mismatch suggests an inherent flaw in the modeling formulation at this critical endpoint, directly leading to the network's inability to achieve an adequate fit in this region.


We identify the root cause of this failure as a fundamental mismatch in the noise scheduling trends between the input and the target. As shown in Fig.~\ref{fig:noise}, the noise amplitude of the network input follows a bridge-like profile, whereas the noise amplitude of the training objective is strictly monotonic. This irreconcilable trend inconsistency is precisely what creates the ill-conditioned learning task, as it forces the network to map a nearly deterministic input to a highly stochastic target at the target endpoint.
Consequently, this creates a critical paradox: at the final refinement stages, when the model should be adding high-fidelity details to converge to the target manifold, its inherent flaw prohibits it from  capping the achievable restoration quality.

To mitigate the endpoint underfitting, we analyze this issue via the stochastic interpolant and decouple it into magnitude and direction failures. We then propose an improved paradigm called Noise-Aligned Diffusion Bridge (NADB), which introduces two key innovations to solve them. First, to resolve the magnitude failure, we reformulate the bridge interpolant, creating a new well-posed learning scheme. Second, to correct the erroneous direction errors, we employ a mean network in a preliminary step. This network constructs the bridge between two distributions that are in closer proximity and alleviating the learning burden on the restoration network.
We validate NADB on a range of challenging image restoration tasks; see Fig.~\ref{fig:zhutu}. Extensive experiments demonstrate that our paradigm significantly outperforms pioneering diffusion bridge models. Notably, our approach achieves superior perceptual quality and reconstruction fidelity, proving the efficacy of our solution. Our contributions include:
\begin{itemize}
\item We identify and analyze a critical endpoint underfitting in existing diffusion bridge models, tracing its root cause to a fundamental noise level mismatch, which manifests as a dual failure in prediction magnitude and direction. 
\item We propose a Noise-Aligned Diffusion Bridge (NADB) that resolves this endpoint underfitting issue. NADB introduces a novel magnitude-aligned stochastic interpolant and a mean network to correct the directional error.
\item We provide extensive experimental validation on ImageNet, demonstrating that our method of resolving these dual errors translates to state-of-the-art performance, significantly outperforming existing diffusion bridge models in both perceptual quality and fidelity.
\end{itemize}

%% file: sec/relatedwork.tex
\begin{figure}[t!]
    \centering
        \includegraphics[width=0.47\textwidth]{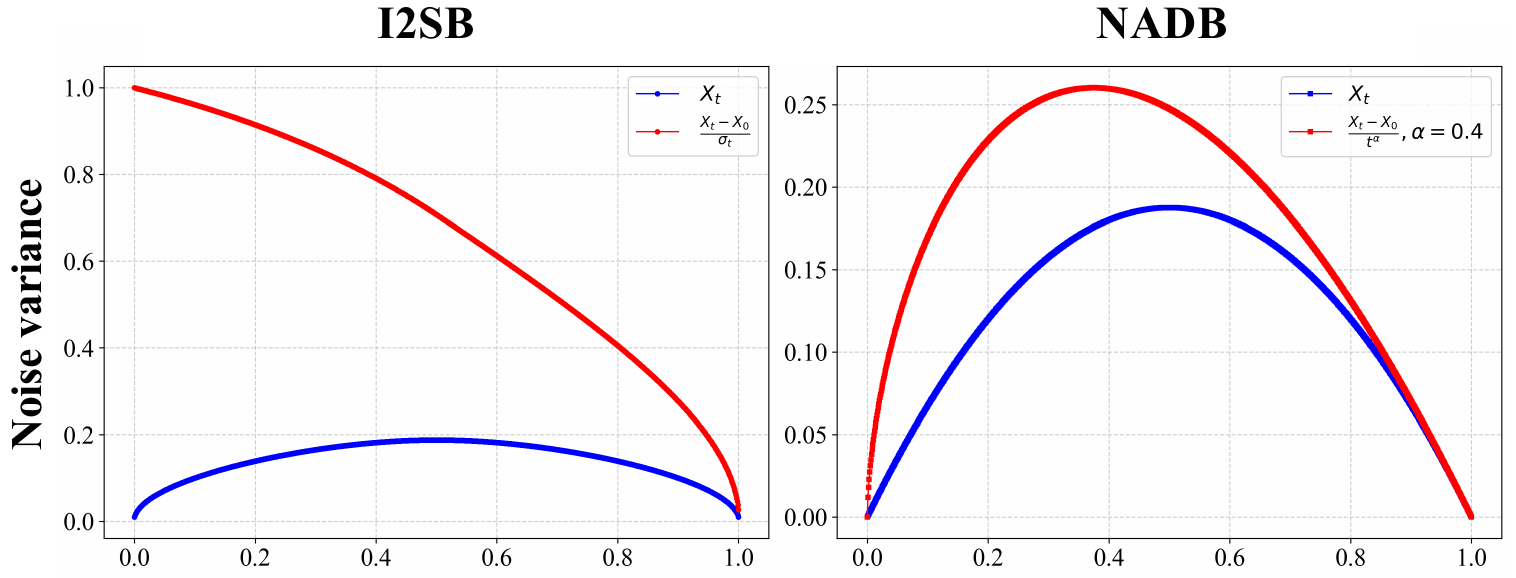}
        \caption{\textbf{Left}: Noise level mismatch in I2SB's training objective defined in Eq.~(\ref{eq:i2sb_loss_final}). The network input $X_t$ exhibits a bridge-like noise profile, while the training  objective $\frac{X_t - X_0}{\sigma_t}$  is monotonically decreasing. This creates a severe learning challenge as $t \to 0$, forcing the network to map a deterministic input to a highly stochastic target. \textbf{Right}: Magnitude aligned noise level in NADB.}
        \label{fig:noise}
\end{figure}

\section{Related work}
\label{sec:related work}

\paragraph{Image Restoration}  aims to recover high-quality images from degraded observations. Its dominant paradigm has evolved from optimizing simple distortion metrics to generating high-fidelity, perceptually-rich results. This progression is marked by three main stages. First, Convolutional Neural Networks (CNNs) trained with pixel-level losses~\cite{dong2015image,dong2016accelerating} excelled on distortion metrics (e.g., PSNR) but produced over-smoothed results, as noted by the perception-distortion trade-off~\cite{blau2018perception,jinjin2020pipal}. Subsequently, Generative Adversarial Networks (GANs)~\cite{ledig2017photo,wang2018esrgan} improved realism using adversarial losses to generate finer textures, but often at the cost of hallucinated artifacts not faithful to the ground truth. Furthermore, their training instability and limited capacity to fully capture the data distribution, often leading to mode collapse, paved the way for more stable and powerful generative frameworks.

\paragraph{Denoising Diffusion Models} are currently the state-of-the-art generative models, achieving superior fidelity and perceptual quality. For restoration tasks, they are typically applied by conditioning the reverse process~\cite{kawar2022denoising,wang2022zero} or focusing on the posterior mean~\cite{ohayon2024posterior}. General research on diffusion models has rapidly advanced the field, including new perspectives on SDE~\cite{song2020score,karras2022elucidating}, optimized noise schedules~\cite{nichol2021improved,karras2022elucidating,peebles2023scalable}, faster samplers~\cite{song2020denoising,lu2022dpm,zhang2022fast}, and improved guidance~\cite{dhariwal2021diffusion,ho2022classifier}.

\paragraph{Diffusion Bridge Model} is an alternative to standard conditional diffusion framework~\cite{de2021diffusion,chen2022likelihood}, which learns a direct stochastic trajectory from the degraded to the clean distribution~\cite{sarkka2019applied}. Its research approach mainly involves modeling the diffusion bridge model as SDE, and then introducing a generative model to fit the score function in the stochastic differential equation~\cite{de2021diffusion,heng2025simulating,liu2022let,liu20232,ICLR2024_20e45668,wang2025implicit,song2020denoising}. Some of them construct bridge models using the h-transform from the diffusion process~\cite{liu2022let,ICLR2024_20e45668}, some completely imitate the forward and backward sampling processes of the diffusion process~\cite{liu20232}, and some focus on the corresponding acceleration algorithms~\cite{wang2025implicit,song2020denoising}. 

\paragraph{Discussion.} We observe that research on diffusion bridge is largely imitating the improvement paths of standard diffusion models. However, as we identified in \S\ref{sec:intro}, methods that directly imitate the diffusion process, such as I2SB~\cite{liu20232}, inherit fundamental flaws, namely the anomalous underfitting at the target endpoint. Instead of adopting the standard score-matching approach, we construct our bridge from the more flexible perspective of Stochastic Interpolants~\cite{albergo2023stochastic}, which enables us to redesign the mapping and resolve this critical endpoint failure.

%% file: sec/preliminaries.tex
\section{Preliminaries}
\label{sec:Preliminaries}
In this section, we introduce several key concepts and notations. We begin with the general definition of a stochastic interpolant, which provides a framework for connecting two distributions using continuous-time processes.

\begin{definition}[Stochastic Interpolant]
\label{def1}
Let $(X_0, X_1)$ be a paired sample from a joint distribution $(\rho_0, \rho_1)$, the stochastic interpolant $\{X_t\}_{t \in [0,1]}$ is defined as:
\begin{equation}
	X_t=\alpha_t X_0+\beta_t X_1+\gamma_t  Z,
  \label{eq:xt_general}
\end{equation}
where $Z \sim \mathcal{N}(0,I)$, $\alpha_t+\beta_t=1$, with boundary conditions $(\alpha_0, \beta_0, \gamma_0) = (1, 0, 0)$ and $(\alpha_1, \beta_1, \gamma_1) = (0, 1, 0)$.
\end{definition}
In context of image restoration tasks, $X_0$ represents the high-quality image, and $X_1$ is its corresponding degraded version. The variable $X_t$ thus traverses a stochastic path connecting the two image manifolds. I2SB~\cite{liu20232} is a specific instance of stochastic interpolant built for diffusion bridge.

\begin{definition}[Diffusion Schr{\"o}dinger Bridge] 
\label{def2}
Using a specific set of coefficients for Eq.~(\ref{eq:xt_general}),
I2SB~\cite{liu20232} derives a path for the tractable Schr{\"o}dinger Bridge:
\begin{equation}
	X_t=\underbrace{\frac{\bar{\sigma}^2_t}{\bar{\sigma}^2_t + \sigma^2_t} X_0+\frac{\sigma^2_t}{\bar{\sigma}^2_t + \sigma^2_t}X_1}_{\text{Mean term}}+\underbrace{\sqrt{\frac{\sigma^2_t \, \bar{\sigma}^2_t}{\bar{\sigma}^2_t + \sigma^2_t}} Z}_{\text{Stochastic term}},
	\label{eq:i2sbxt}
\end{equation}
where $\sigma^2_t=\int_0^t \beta_\tau d\tau$ and $\bar{\sigma}^2_t = \int_t^1 \beta_\tau d\tau$. As shown in Fig.~\ref{fig:c}, the network $\epsilon(X_t, t; \theta)$ is trained by mimicking the standard diffusion objective, which for this specific path simplifies to the following regression loss:
\begin{equation}
    \mathcal{L}_{\text{I2SB}} = \mathbb{E}_{t, X_0, X_1, Z}\left[ \left\| \epsilon(X_t, t; \theta) - \frac{X_t-X_0}{\sigma_t} \right\|^2 \right].
    \label{eq:i2sb_loss_final}
\end{equation}
During inference, I2SB estimates $X_{t}$ in reverse from $t=1$ to $t=0$, analogous to standard diffusion samplers.
\end{definition}

The coupling of this specific interpolant path in Eq.~(\ref{eq:i2sbxt}) and final training objective in Eq.~(\ref{eq:i2sb_loss_final}) is the crucial point for our analysis.
The training objective is forced to approximate pure gauss noise as $t \to 0$, 
precisely when the interpolant $X_t$ is converging to the clean data $X_0$. As we will investigate, this  formulation is suboptimal for building diffusion Bridge and leads to the underfitting phenomena.

%% file: sec/method.tex
\section{Methodology}


\label{sec:Method}

\begin{figure}[t!]
    \centering
        \includegraphics[width=0.47\textwidth]{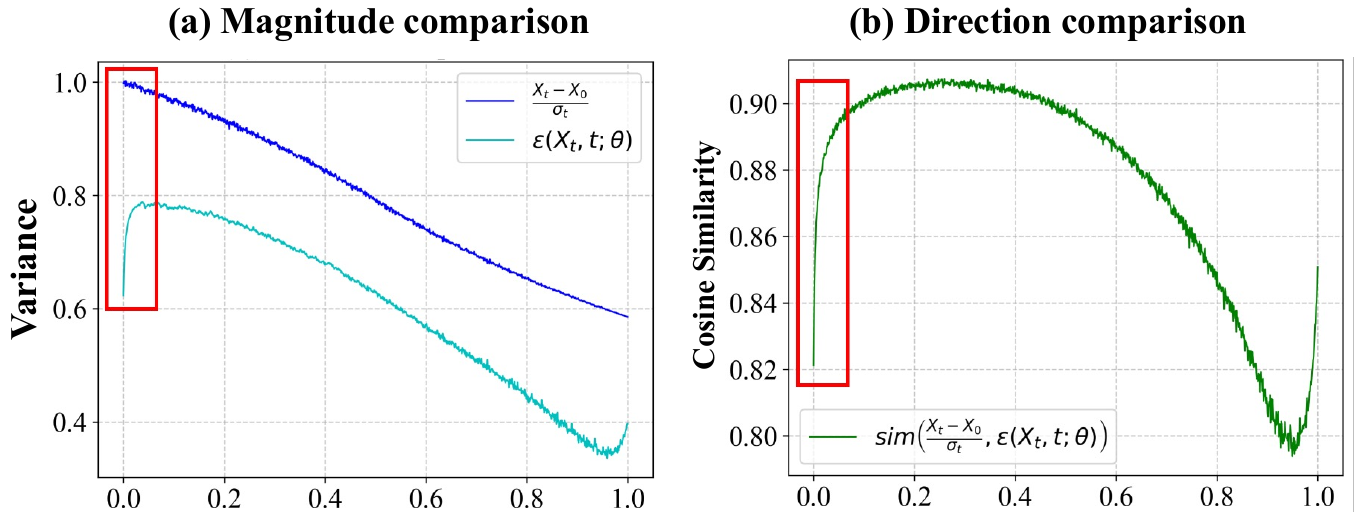}
        \caption{Empirical analysis of I2SB's endpoint underfitting. The noise mismatch shown in Fig.~\ref{fig:noise} causes the training to fail as $t \to 0$ in both magnitude (variance collapse) and direction (cosine similarity drop), confirming a severe underfitting phenomenon.}
        \label{fig:bluruni}
\end{figure}

\subsection{Noise Level Mismatch in Diffusion Bridge }
As shown in Fig.~\ref{fig:noise}, we first show that there exists an issue of noise level mismatch in diffusion bridge models, which causes underfitting problem at the target endpoint in Fig.~\ref{fig:bluruni}. 
To analyze this mismatch, we expand the training  objective  based on the definition in Eq.~(\ref{eq:i2sbxt}):
\begin{equation}
	\frac{X_t-X_0}{\sigma_t}=\underbrace{\frac{\sigma_t}{\bar{\sigma}^2_t + \sigma^2_t}(X_1-X_0)}_{\text{Mean term}}+\underbrace{\frac{\bar{\sigma}_t}{\sqrt{\bar{\sigma}^2_t + \sigma^2_t}}  Z}_{\text{Stochastic term}}.
	\label{eq:targeti2sb}
\end{equation}
While the stochastic term is analogous to the standard diffusion objective of predicting pure noise, the mean term represents a fundamental departure by introducing a deterministic signal component that the original diffusion framework lacks.
We now clearly compare the terms of the network input Eq.~(\ref{eq:i2sbxt}) and the network target Eq.~(\ref{eq:targeti2sb}). 

\paragraph{Noise Level Mismatch} As illustrated in Fig.~\ref{fig:noise}, a significant discrepancy exists at the target endpoint ($t \to 0$): For the network input ($X_t$), its noise coefficient approaches 0 as $t \to 0$, making the input become nearly deterministic $X_0$.
For the training  target, its noise coefficient approaches 1 as $t \to 0$. The target becomes highly stochastic $Z$. This mismatch imposes an ill-conditioned task on the network: it predicts random noise from a clean and deterministic input.

\paragraph{Quantitative Metrics}
To quantify the network's failure on this ill-conditioned task, we assess the difference between the network's prediction and the ground truth target along two dimensions:
\begin{itemize}
     \item \textbf{Magnitude:} We compare the Variance between the two. This measures whether the network predicts the correct amplitude of stochasticity.
    \item \textbf{Direction:} We compute the Cosine Similarity between the two. This measures whether the direction of the prediction is pointing in the correct direction.
\end{itemize}

\paragraph{Endpoint Underfitting} We show the empirical results for I2SB~\cite{liu20232} on image restoration tasks in Fig.~\ref{fig:bluruni}, using a randomly selected ImageNet sample as an example.
These plots clearly demonstrate that as $t \to 0$, the network's prediction fails on both dimensions: The predicted variance sharply collapses, failing to match the high variance of the target. The cosine similarity also precipitously drops, indicating that the prediction is pointing in an erroneous direction.
We term this dual failure Endpoint Underfitting. This flaw is critical because it occurs during the final, most crucial refinement stages of generation, directly limiting the achievable restoration quality.

This ill-conditioned mapping relationship poses a significant challenge to the network's fitting capacity in this boundary region.  
Therefore, improving the network's fitting capability near the target endpoint by redesigning the mapping relationship is imperative.  We address this by proposing a means network and a noise correspondence strategy.  Our experiments verify that the mean network effectively enhances the output directional accuracy near the target endpoint, while the noise correspondence method successfully resolves the variance discrepancy issue.

\begin{figure}[t!]
    \centering
        \includegraphics[width=0.47\textwidth]{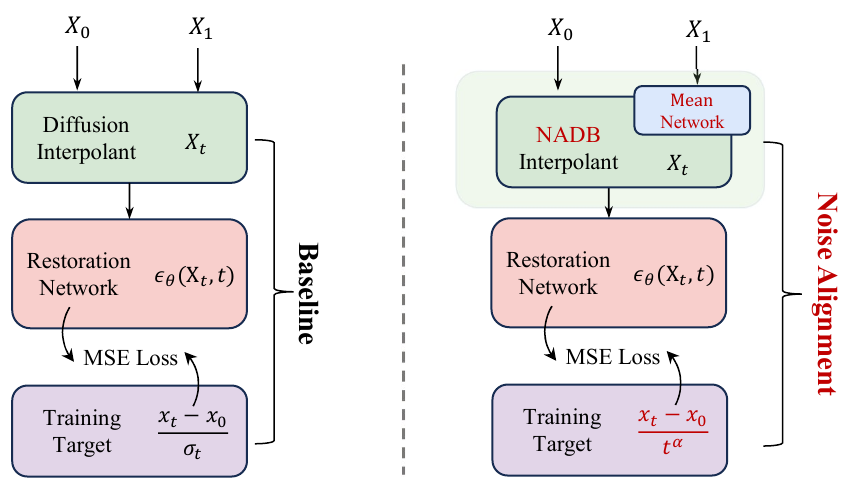}
        \caption{\textbf{Left}: The training pipeline of I2SB. It uses score matching to train the network. \textbf{Right}: The training pipeline of our proposed NADB. The process involves two key components: (1) A frozen Mean Network pre-processes the degraded image $X_1$ to estimate the posterior mean $\mathbb{E}[X_0|X_1]$, which replaces the original $X_1$ as a proxy endpoint to construct the new stochastic interpolant $X_t$. (2) The Restoration Network $\epsilon_{\theta}$ is then trained to predict our proposed magnitude-aligned target $\frac{X_t-X_0}{t^\alpha}$ from this new $X_t$.}
        \label{fig:c}
\end{figure}

\subsection{Magnitude-Aligned Diffusion Bridge}
To solve the endpoint underfitting, we redesign the stochastic interpolant path and training objective to be ``magnitude-aligned'' . That is, the noise magnitude of the network's input and its  training objective should be coupled, ideally both vanishing at the endpoints. We explicitly design a new interpolant and a target to resolve this magnitude mismatch.

\begin{definition}[Magnitude-Aligned Stochastic Interpolant]
\label{def:madbm}
Let $\alpha \in (0,1)$ and $k$ be a finite constant. We define the {magnitude-aligned stochastic interpolant} $X_t$ connecting the target $X_0 \sim \rho_0$ and the degraded $X_1 \sim \rho_1$ as:
\begin{equation}
	X_t \coloneqq (1-t^{\alpha})X_0+t^{\alpha}X_1+{kt(1-t)}  Z.
    \label{eq:madbm_x}
\end{equation}
 We define the corresponding {magnitude-aligned training objective} $Y_t$ as the scaled displacement, which decomposes as:
 \begin{equation}
   Y_t \coloneqq \frac{X_t-X_0}{t^{\alpha}} = (X_1-X_0)+{kt^{1-\alpha}(1-t)}  Z.
     \label{eq:madbm_target}
 \end{equation}
The training objective is to predict this target:
 \begin{equation}
   \mathcal{L}_{\text{base}} = \mathbb{E}_{t, X_0, X_1, Z}\left[ \left\| \epsilon(X_t, t; \theta) - Y_t \right\|^2 \right].
   \label{eq:madbm_loss}
\end{equation}
\end{definition}

This construction is specifically chosen to satisfy our endpoint alignment constraints, which we formalize in the following proposition.

\begin{proposition}[Endpoint Noise Alignment]
\label{prop:alignment}
Let the interpolant $X_t$ and the target $Y_t$ be defined as in Def.~\ref{def:madbm}. Let $\gamma_X(t) = kt(1-t)$ be the noise coefficient of the input $X_t$, and $\gamma_Y(t) = kt^{1-\alpha}(1-t)$ be the noise coefficient of the target $Y_t$ from Eq.~(\ref{eq:madbm_target}).
Given $\alpha \in (0, 1)$, both noise coefficients vanish at the endpoints $t=0$ and $t=1$:
\begin{gather}
    \lim_{t \to 0} \gamma_X(t) = 0 \quad \text{and} \quad \lim_{t \to 1} \gamma_X(t) = 0, \\
    \lim_{t \to 0} \gamma_Y(t) = 0 \quad \text{and} \quad \lim_{t \to 1} \gamma_Y(t) = 0.
\end{gather}
\end{proposition}
\begin{proof}
The proof follows from a direct evaluation of the polynomial terms $\gamma_X(t)$ and $\gamma_Y(t)$, given that $\alpha \in (0,1)$.
\end{proof}

This new formulation successfully resolves the noise magnitude mismatch. As visualized in Fig.~\ref{fig:noise} according to Prop.~\ref{prop:alignment}, the noise terms in the network's input $X_t$ and its target $Y_t$ remain within the same order of magnitude throughout the entire interval.

\subsection{Mean Network for Direction Alignment}

\label{sec:mean_network}

While Prop.~\ref{prop:alignment} solves the magnitude mismatch, the target $Y_t$ in Eq.~(\ref{eq:madbm_target}) retains a deterministic displacement term $(X_1-X_0)$, which is fundamentally different from the pure noise target of standard diffusion models.  A large distributional gap between the degraded $X_1$ and clean $X_0$ makes this term difficult to regress, which can lead to directional errors.
We introduce a Mean Network to map the degraded $X_1$ to a cleaner, intermediate representation $\hat{X}_0$, which is significantly closer to the target $X_0$. We then construct our diffusion bridge between $X_0$ and this new, closer endpoint $\hat{X}_0$.

\begin{definition}[Mean Network]
\label{def:mean_network}
Given the joint distribution of the paired data $(X_0, X_1) \sim (\rho_0, \rho_1)$, we define a mean network $M(\cdot; \phi)$ trained to approximate the posterior mean $\mathbb{E}[X_0 | X_1]$. The network's output is defined as:
\begin{equation}
\hat{X}_0 = M(X_1; \phi).
\label{eq:M}
\end{equation}
The parameters $\phi$ are optimized by minimizing the mean squared error (MSE) objective:
\begin{equation}
\mathcal{L}_{\text{MSE}}(\phi) = \mathbb{E}_{(X_0, X_1)} \left[ \left\|  M(X_1; \phi) - X_0 \right\|^2 \right].
\label{eq:mean_loss}
\end{equation}
\end{definition}
This loss guides $M_\phi(X_1)$ to converge to the posterior mean $\mathbb{E}[X_0 | X_1]$, which we denote as $\hat{X}_0$ (forming distribution $\hat{\rho}_0$). While this $\hat{X}_0$ may be an over-smoothed representation~\cite{blau2018perception}, its role is twofold: First, it corrects the \textit{directional} error by simplifying the regression task from predicting the complex displacement $(X_1 - X_0)$ to the simpler $(\hat{X}_0 - X_0)$. Second, it provides a cleaner bridge endpoint $\hat{\rho}_0$ that is significantly closer to the target $\rho_0$ than the original $\rho_1$. Our core strategy is to construct the bridge between this closer pair $(\rho_0, \hat{\rho}_0)$. The following theorem provides the theoretical justification for this proximity benefit.

\begin{theorem} \label{thm:mean}
Let $ \rho_0$ be the target distribution, $\rho_1$ the degraded distribution, and $\hat{\rho}_0$ be the output distribution of the mean network $M_\phi$ optimized via Eq.~(\ref{eq:mean_loss}). The mean network reduces the distributional gap in the Wasserstein-2 sense:
\begin{equation}
W_2(\rho_0,\hat{\rho}_0) \leq W_2(\rho_0,\rho_1).
\end{equation}
\end{theorem}

The proof for Theorem~\ref{thm:mean} is provided in Supplementary Materials. Our experimental results confirm this theoretical motivation: establishing the diffusion bridge between distributions with a smaller initial divergence directly translates into more stable and accurate fitting.

\subsection{Final NADB Formulation}

With both components established, we define the full Noise-Aligned Diffusion Bridge (as shown in Fig.~\ref{fig:c}), which first computes the posterior mean $\hat{X}_0 =  M(X_1; \phi)$, then builds magnitude-aligned bridge between the target $X_0$ and this cleaner endpoint $\hat{X}_0$. The final NADB stochastic interpolant is:
\begin{equation}
X_t \coloneqq (1-t^{\alpha})X_0+t^{\alpha}\hat{X}_0+{kt(1-t)}  Z.
\label{eq:nadb_x_final}
\end{equation}
The network $\epsilon(X_t, t; \theta)$ is then trained to predict the corresponding target $Y_t = (\hat{X}_0-X_0)+{kt^{1-\alpha}(1-t)}  Z$ using the final objective:
\begin{equation}
\mathcal{L}_{\text{NADB}} = \mathbb{E}_{t, X_0, X_1, Z}\left[ \left\| \epsilon(X_t, t; \theta) - \frac{X_t-X_0}{t^{\alpha}} \right\|^2 \right].
\label{eq:nadb_loss_final}
\end{equation}
This training process is summarized in Algorithm~\ref{alg:A1}.
\begin{algorithm}[!t]
\caption{NADB Training}
\label{alg:A1}
\renewcommand{\algorithmicrequire}{\textbf{Input:}}
\renewcommand{\algorithmicensure}{\textbf{Output:}}
\begin{algorithmic}[1]
\Require Paired data distributions $(\rho_0, \rho_1)$, pre-trained mean network $M_\phi$
\Repeat
\State $t\sim \mathcal{U}([0,1])$, $X_0\sim \rho_0$, $X_1\sim \rho_1$
\State $\hat{X}_0 = M_\phi(X_1)$ \Comment{Get cleaner endpoint via Eq.~(\ref{eq:M})}
\State $X_t \sim q(X_t|X_0,\hat{X}_0)$ \Comment{Sample via Eq.~(\ref{eq:nadb_x_final})}
\State Take gradient descent step on $\epsilon(X_t,t;\theta)$ using $\mathcal{L}_{\text{NADB}}$ in Eq.~(\ref{eq:nadb_loss_final})
\Until converges
\end{algorithmic}
\end{algorithm}

\paragraph{Sampling} 
The corresponding reverse sampling procedure is derived in detail in Supplementary Materials. A key property is that the reverse process naturally splits into two distinct stages at a time threshold $d \approx \frac{1-\alpha}{2-\alpha}$. This split is necessary to maintain a non-positive variance term near the $t \to 0$ endpoint. This approach provides flexibility, enabling a separate, endpoint-conditioned transition for $t \in [0, d]$. The full generation process is summarized in Algorithm~\ref{alg:A2}.
\begin{algorithm}[!t]
\caption{NADB Generation}
\label{alg:A2}
\renewcommand{\algorithmicrequire}{\textbf{Input:}}
\renewcommand{\algorithmicensure}{\textbf{Output:}}
\begin{algorithmic}[1]
\Require Degraded input $X_1 \sim \rho_1$, time threshold $d$, steps $N$
\State $\hat{X}_0 = M_\phi(X_1)$ \Comment{Compute clean endpoint once}
\For{$n=N$ to $1$} 
\State Predict $X^{\epsilon}_0$ using $\epsilon(X_t,t;\theta)$
\If{$t_n < d$}
\State Sample from $X_{n-1}\sim p(X_{n-1}\mid X^{\epsilon}_0,\hat{X}_0,X_{n})$ 
\Else
\State Sample from $X_{n-1} \sim p(X_{n-1}\mid X^{\epsilon}_0,X_{n})$ 
\EndIf
\EndFor \\
\Return $X_0$
\end{algorithmic}
\end{algorithm}

%% file: sec/experiment.tex
\section{Experiments}
\label{sec:experiments}

\subsection{Implementation Details}

\begin{table*}
  \caption{Comparison with I2SB on numerous image restoration tasks.}
  \label{tab:all}
  \centering
  \footnotesize
  \renewcommand{\arraystretch}{0.85}
  \setlength{\tabcolsep}{2pt} 
  \begin{tabular}{@{}l|cccccc|cccccc|cccccc@{}}
    \toprule
\multicolumn{1}{c|}{\textbf{Method}} &   \multicolumn{6}{c|}{\textbf{JPEG Restoration}} & \multicolumn{6}{c|}{\textbf{4$\times$ Super-Resolution}} & \multicolumn{6}{c}{\textbf{Deblurring}} \\
    \midrule
    &QF&NFE&FID$\downarrow$&PSNR$\uparrow$& SSIM$\uparrow$&LPIPS$\downarrow$&Filter&NFE&FID$\downarrow$&PSNR$\uparrow$&SSIM$\uparrow$&LPIPS$\downarrow$&Kernel&NFE&FID$\downarrow$&PSNR$\uparrow$&SSIM$\uparrow$&LPIPS$\downarrow$\\
\cmidrule(lr){2-19}
I2SB&\multirow{4}{*}{5}&\multirow{2}{*}{10}&8.0&24.50&0.69&0.30&\multirow{4}{*}{Pool}&\multirow{2}{*}{10}&7.3&24.86&0.70&0.27&\multirow{4}{*}{Uniform}&\multirow{2}{*}{10}&10.3&24.19&0.65&0.32 \\ 
NADB&                  &                   &6.9&24.45&0.69&0.30&                     &                   &5.3&24.75&0.71&0.23&                        &                   &4.8&27.70&0.81&0.18  \\
 \cmidrule(r){1-1} \cmidrule(lr){3-7} \cmidrule(lr){9-13} \cmidrule(lr){15-19} 
I2SB&                  &\multirow{2}{*}{100}&4.7&23.60&0.65&0.31&                    &\multirow{2}{*}{100}&4.1&23.57&0.65&0.26&                       &\multirow{2}{*}{100}&5.0&22.47&0.58&0.32 \\ 
NADB&                  &                    &4.3&23.63&0.65&0.30&                    &                    &3.7&23.63&0.67&0.22&                       &                    &3.4&27.39&0.80&0.17  \\
\midrule
I2SB&\multirow{4}{*}{10}&\multirow{2}{*}{10}&6.1&26.27&0.76&0.24&\multirow{4}{*}{Bicubic}&\multirow{2}{*}{10}&8.1&25.07&0.70&0.27&\multirow{4}{*}{Gaussian}&\multirow{2}{*}{10}&7.4&25.42&0.71&0.27 \\ 
NADB&                  &                   &5.7&26.55&0.77&0.23&                     &                   &6.8&25.39&0.73&0.24&                        &                   &4.2&30.03&0.87&0.15  \\
 \cmidrule(r){1-1} \cmidrule(lr){3-7} \cmidrule(lr){9-13} \cmidrule(lr){15-19} 
I2SB&                  &\multirow{2}{*}{100}&3.6&25.34&0.72&0.24&                    &\multirow{2}{*}{100}&4.1&23.74&0.65&0.27&                       &\multirow{2}{*}{100}&3.9&23.98&0.66&0.27 \\ 
NADB&                  &                    &3.5&25.81&0.74&0.23&                    &                    &4.1&24.46&0.70&0.22&                       &                    &3.1&30.34&0.88&0.13  \\ 
    \bottomrule
  \end{tabular}
\end{table*}

\begin{figure*}[htbp]
  \centering
  \includegraphics[width=\textwidth]{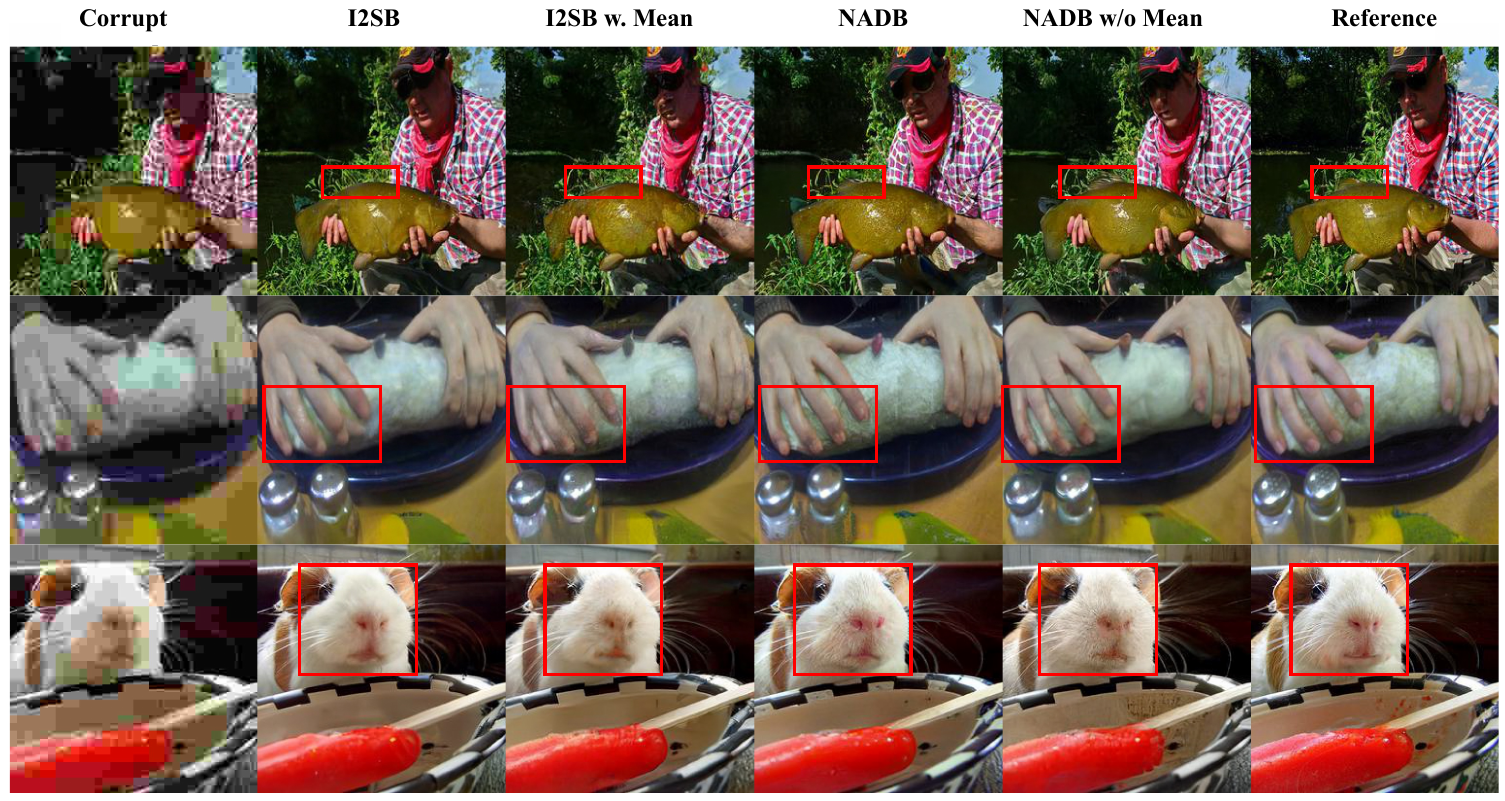}
  \caption{Ablation comparison on the JPEG-5 restoration task. We compare I2SB, I2SB augmented with our Mean Network (``w. Mean''), NADB without the Mean Network (``w/o Mean''), and the proposed NADB. }
  \label{fig:xiaorong}
\end{figure*}

\begin{figure*}[t!]
    \centering
        \includegraphics[width=0.95\textwidth]{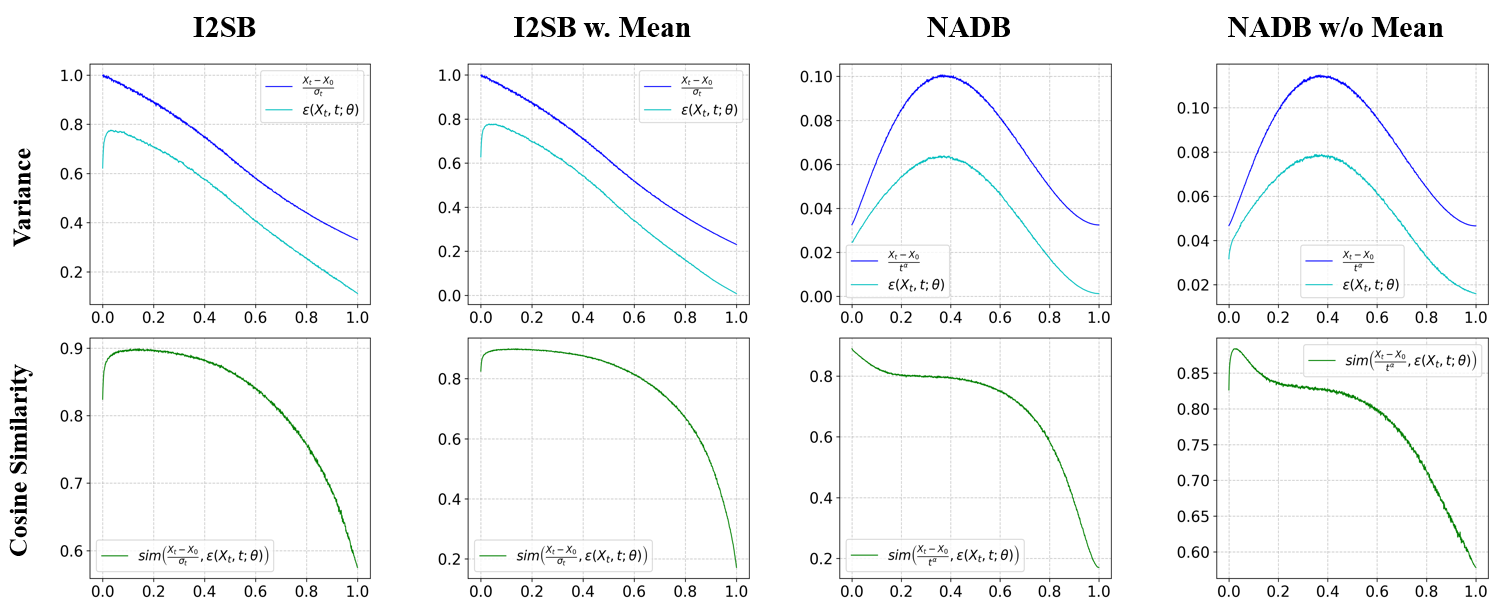}
        \caption{Endpoint performance comparison for the ablated models across the time step $t$ (horizontal axis).}
        \label{fig:xrtubiao}
\end{figure*}

\paragraph{Network Architecture and Training.}
Our framework consists of two networks with identical U-Net~\cite{ronneberger2015u} architectures: the main {Restoration Network} $\epsilon(X_t, t; \theta)$ and the {Mean Network} $M(X_1; \phi)$. Most of our implementation and hyperparameters follow I2SB~\cite{liu20232} for a fair comparison. 
Specifically, we initialize the two U-Nets using the publicly available ADM checkpoint~\cite{dhariwal2021diffusion} (pre-trained on ImageNet 256$\times$256). The Mean Network $M_\phi$ is first trained independently on each task until its MSE loss converges,  and the time-step input to $M_\phi$ is held constant with zero. 
Both networks are trained using Adam optimizer with a learning rate of $1 \times 10^{-4}$ and a batch size of $256$. For NADB, we set the hyperparameters $\alpha=0.4$ and $k=0.75$. 
All experiments were conducted on 8 NVIDIA A100 GPUs.

\paragraph{Baselines.}
We compare our method against three categories of generative models (comparisons with some baselines are provided in the Supplementary Materials):
\begin{itemize}
    \item \textbf{Conditional Diffusion Models:} We benchmark against established restoration methods, including DDRM~\cite{kawar2022denoising}, DDNM~\cite{wang2022zero},  $\Pi$GDM~\cite{song2023pseudoinverse}, Palette~\cite{saharia2022palette}, DiT4SR~\cite{Duan2025DiT4SRTD} and RDDM~\cite{Liu_2024_CVPR}.
    \item \textbf{Score-Based Diffusion Bridges:} We conduct a comparison against I2SB~\cite{liu20232}, I3SB~\cite{wang2025implicit},  DDBM~\cite{ICLR2024_20e45668}, RDBM~\cite{Wang2025ResidualDB}, and GOUB~\cite{Yue2023ImageRT}. For I2SB, we conduct a comprehensive evaluation  to isolate the improvements from our proposed formulation.  I2SB maintains the same training steps as NADB.
\end{itemize}
\paragraph{Datasets.} We compare methods following the benchmarks from I2SB.
For image restoration, we evaluate on three tasks: JPEG artifact removal, deblurring, and 4$\times$ super-resolution (64$\times$64 $\to$ 256$\times$256). All models are evaluated on the 256$\times$256 ImageNet dataset~\cite{deng2009imagenet}.
For image translation, we use the standard edges$\to$handbags and edges$\to$shoes datasets, evaluated at 64$\times$64 resolution.

\paragraph{Evaluation.} For restoration, we sample from 10,000 images from the ImageNet validation set, and FID is computed on full validation set. For translation, due to the small validation set, FID is computed on the training set, following standard practice.
To ensure a fair comparison, all experimental settings were kept consistent with the baseline methods.
Models without pre-trained weights were faithfully reproduced from their published configurations.

\begin{table}
  \caption{Comparison of NADB with diffusion models (NFE=100). Reported results are taken from~\cite{liu20232}, and 4$\times$ Super-Resolution is evaluated on the full validation set, differing from Tab.~\ref{tab:all}.}
  \label{tab:cwdm}
  \centering
  \scriptsize 
  \renewcommand{\arraystretch}{0.85}
  \setlength{\tabcolsep}{3pt}
  \begin{tabular}{@{}lcc|lcc|lcc@{}}
    \toprule
    \multicolumn{3}{c|}{\textbf{JPEG Restoration}} & \multicolumn{3}{c}{\textbf{4$\times$ Super-Resolution}} &
    \multicolumn{3}{|c}{\textbf{Deblurring}} \\
    \cmidrule(lr){1-3} \cmidrule(lr){4-6}\cmidrule(lr){7-9}
    QF & Method & FID$\downarrow$ & Filter & Method & FID$\downarrow$ & Kernel & Method & FID$\downarrow$ \\
    \midrule
    \multirow{4}{*}{5} 
      & DDRM & 28.2 & \multirow{4}{*}{Pool} 
      & DDRM & 14.8 & \multirow{4}{*}{\makecell{Uniform}} 
      & DDRM & 9.9 \\ 
      & $\Pi$GDM & 8.6 & & DDNM & 9.9 & & DDNM & 3.0 \\
      & Palette & 8.3 & &$\Pi$GDM & 3.8 & &Palette &4.1 \\
      & NADB & 4.3 & & NADB & 1.1 & & NADB & 3.4 \\
      \midrule
    \multirow{4}{*}{10}
      & DDRM & 16.7 & \multirow{4}{*}{\makecell{Bicubic}}
      & DDRM & 21.3 & \multirow{4}{*}{\makecell{Gaussian}} 
      & DDRM & 6.1 \\
      & $\Pi$GDM & 6.0 & & DDNM & 13.6 & & DDNM & 2.9 \\
      & Palette & 5.4 & & $\Pi$GDM & 3.6 & & Palette & 3.1 \\
      & NADB & 3.5 & & NADB & 1.0 & & NADB & 3.1 \\
    \bottomrule
  \end{tabular}
\end{table}

\subsection{Comparisons}

\paragraph{Comparison with Conditional Diffusion Models}
Tab.~\ref{tab:cwdm} presents a comparison against leading conditional diffusion models for restoration. We observe that while these methods perform well on distortion metrics, their outputs can exhibit blurring. This reflects a known perception-distortion trade-off~\cite{saharia2022palette}, where pixel-wise losses favor over-smoothed results. Given that generative models are optimized for perceptual quality, we follow standard practice~\cite{saharia2022palette} in using FID as the primary metric for comparing. The results demonstrate the strong performance of our bridge-based model against these general-purpose diffusion solvers.

\begin{figure}[t]
    \centering
    \includegraphics[width=0.8\linewidth]{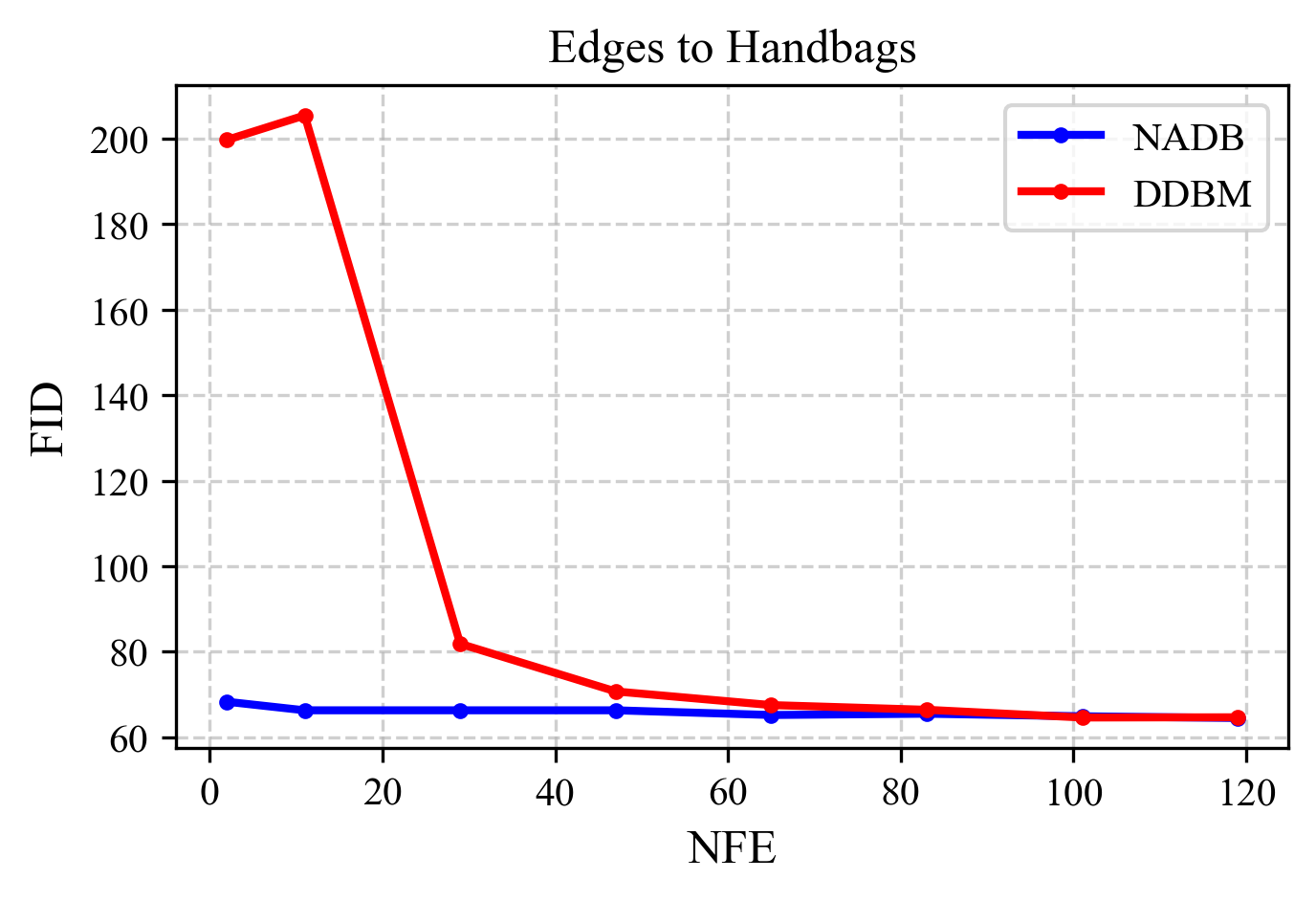}
    \caption{Comparison of FID changing trends. }
    \label{fig:fid_nfe}
\end{figure}

\paragraph{Comparison with I2SB}
We conduct a head-to-head comparison with I2SB~\cite{liu20232}, the baseline our work directly improves upon. Both models were trained with an equivalent computational budget for a fair comparison. As shown in Tab.~\ref{tab:all}, NADB consistently  outperforms I2SB across multiple image restoration tasks. 
This superiority is evident in the vast majority of both perceptual and distortion metrics, confirming the effectiveness of our proposed modifications.

\paragraph{Generalization to Image-to-Image Translation}
\begin{figure*}[t!]
    \centering
        \includegraphics[width=0.8\textwidth]{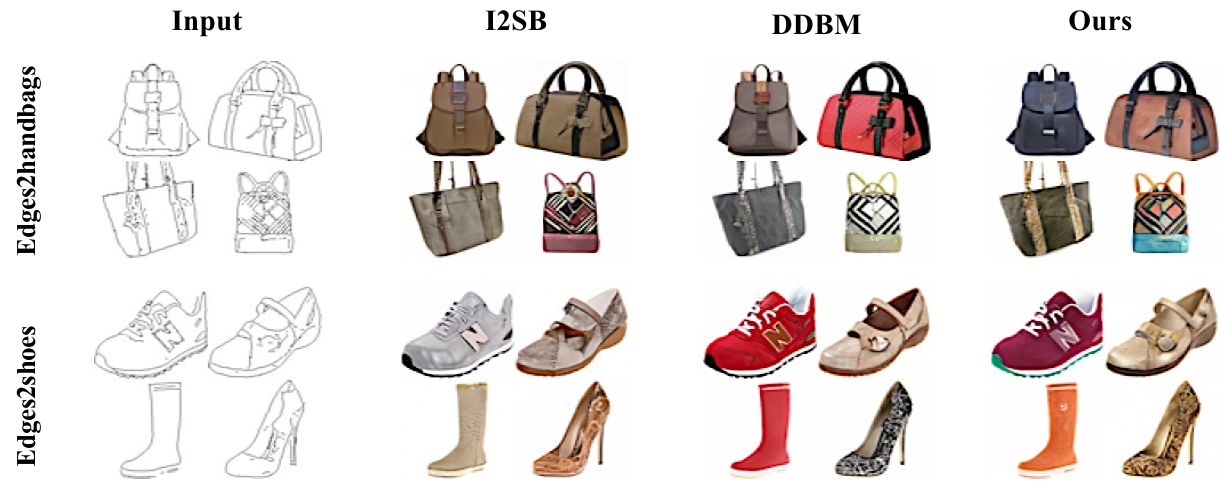}
        \caption{Image translation comparison. All images are in $64 \times 64$ resolution.}
        \label{fig:e2h_e2s}
\end{figure*}
To evaluate the generalization capability of NADB, we conducted a systematic assessment on standard image-to-image translation tasks at $64\times64$ resolution,  as illustrated in Fig.~\ref{fig:e2h_e2s}. We compared NADB against I2SB and DDBM, a strong baseline for image translation. As shown in Tab.~\ref{tab:trans}, our method outperforms both competing approaches. Notably, Fig.~\ref{fig:fid_nfe} shows that our method maintains high-quality generation with low NFE, while DDBM degrades significantly.

\begin{table}[t]
  \centering
  \caption{Comparison with diffusion bridge models on image translation tasks.}
  \label{tab:trans}
  \footnotesize
  \setlength{\tabcolsep}{3pt}
  \begin{tabular}{@{}l|cccc|cccc@{}}
    \toprule
    & \multicolumn{4}{c}{\textbf{Edges to Handbags}} & \multicolumn{4}{c}{\textbf{Edges to Shoes}} \\
    \cmidrule(lr){2-5} \cmidrule(l){6-9}
    Method & FID$\downarrow$ & PSNR$\uparrow$ & SSIM$\uparrow$ & LPIPS$\downarrow$ & FID$\downarrow$ & PSNR$\uparrow$ & SSIM$\uparrow$ & LPIPS$\downarrow$ \\
    \midrule
    I2SB  & 64.8 & 16.70 & 0.62 & 0.21 & 42.9 & 19.39 & 0.73 & 0.15 \\
    DDBM  & 64.7 & 20.72 & 0.74& 0.10 & 44.0 & 21.21 & 0.77 & 0.09 \\
    NADB  & 64.5 & 18.90& 0.69 & 0.13 & 42.2 & 19.41 & 0.74 & 0.13 \\
    \bottomrule
  \end{tabular}
\end{table}

\subsection{Ablation Study}

Our NADB framework introduces two key innovations to solve the endpoint underfitting phenomenon: a {Mean Network} to correct the directional error, and a {Magnitude-Aligned Interpolant} to resolve the magnitude mismatch. To validate the individual contributions and necessity of both components, we conduct a comprehensive ablation study.

We compare four distinct models on the challenging JPEG-5 restoration task, a benchmark chosen specifically because its high degradation accentuates the endpoint failures we aim to solve:
the original I2SB model~\cite{liu20232}, the I2SB baseline augmented with our Mean Network preprocessing (I2SB w. Mean), our NADB formulation applied directly between $X_0$ and $X_1$ without the Mean Network (Ours w/o Mean), our complete proposed method.
The quantitative results of this ablation are presented in Tab.~\ref{tab:jpeg5xiaorong}. We provide qualitative comparisons in Fig.~\ref{fig:xiaorong} and detailed endpoint performance plots (variance and cosine similarity) for the ablated models in Fig.~\ref{fig:xrtubiao}.
From these results, we draw two key conclusions:
\paragraph{Magnitude alignment alone is insufficient.}  ``NADB w/o Mean'' successfully resolves the variance collapse (magnitude failure), but still exhibits a significant drop in cosine similarity (direction failure) at the endpoint.
\paragraph{Mean network alone is insufficient.}  The ``I2SB w. Mean'' model fails to mitigate the underfitting, with both its variance and direction metrics collapsing. This demonstrates that simply preprocessing the endpoint is not enough if the underlying mapping remains ill-conditioned.

These findings empirically confirm our analysis: the Mean Network is required to correct the prediction direction, while the Magnitude-Aligned Interpolant is required to fix the variance. Both components are essential to fully resolve the endpoint underfitting problem.

\begin{table}
\setlength{\tabcolsep}{4pt}
  \caption{Qualitative comparisons with four models on {JPEG-5} restoration tasks.}
  \label{tab:jpeg5xiaorong}
  \centering
  \footnotesize
  \begin{tabular}{@{}lcccccc@{}}
    \toprule
    QF & NFE&Method & FID $\downarrow$ & PSNR $\uparrow$&SSIM$\uparrow$ & LPIPS $\downarrow$  \\
    \midrule
      \multirow{8}{*}{5}& \multirow{4}{*}{10} &I2SB & 8.0 &24.50 &0.69&0.30  \\
                         &                    & NADB & 6.9&24.45&0.69&0.30\\
                         &                   & I2SB w. Mean& 8.8&24.51&0.69&0.31\\
                          &                    & NADB w/o Mean &7.0&24.36&0.69&0.30\\
                         \cmidrule(lr){2-7}
                        & \multirow{4}{*}{100} &I2SB & 4.7 &23.60 &0.65&0.31  \\
                         &                      &NADB & 4.3 &23.63 &0.65&0.30  \\
                           &                   & I2SB w. Mean&4.4&23.64&0.65&0.31\\
                          &                     & NADB w/o Mean&5.2&23.54&0.65&0.31\\
    
    \bottomrule
  \end{tabular}
\end{table}

%% file: sec/Conclusion.tex
\section{Conclusion}
\label{sec:Clu}

In this paper, we identify a critical endpoint underfitting in diffusion bridge models, rooted in a fundamental noise level mismatch that causes a dual failure in prediction magnitude and direction. 
We propose NADB to resolve this by introducing a Mean Network that corrects directional failure and a Magnitude-Aligned Interpolant that resolves magnitude failure. Experiments confirm NADB resolves underfitting and outperforms prior methods, highlighting the value of coupling the interpolant path with the training \mbox{objective}.

{\sloppy
\paragraph{Acknowledgments.} This paper is supported by the National Natural Science Foundation of China (Nos. U23B2012, 12431012).\par
}

%% file: sec/X_suppl.tex
\clearpage
\setcounter{page}{1}
\maketitlesupplementary

\section{Sampling Procedure}
 Since the training  objective enables the prediction of the initial image $X_0$ from any noisy state $X_t$, we design the corresponding sampling distribution to ensure that samples progressively denoise along the interpolation path.
 Let $\{t_i\}_{i=0}^N$ represent a partition of the unit interval $[0, 1]$, where $t_0 = 0$, $t_N = 1$, and $t_i < t_{i+1}$ for all $i = 0, \dots, N$. Our sampling framework operates by coordinating the use of two distinct conditional distributions: the standard reverse transition $p(X_{t-1} \mid X^{\epsilon}_0, X_t)$ and the endpoint-conditioned distribution $p(X_{t-1} \mid X^{\epsilon}_0, \hat{X}_0, X_t)$. Specifically, we adopt a partitioned approach that applies  different conditional distributions over distinct temporal regions.   The standard reverse transition is used over the interval over the interval $[t_i, t_N]$, while the endpoint-conditioned distribution is employed over the complementary interval $[t_0, t_i]$. This formulation eliminates the intermediate step of solving for $x^{\epsilon}_0$ from the network output $\epsilon_\theta$. We refer to this sampling method as a two-stage sampling strategy, implemented as follows:
 \subsection{The First Stage }
 We use $X_{t_{i-1}}\sim p(X_{t_{i-1}}\mid X^{\epsilon}_0,X_{t_i})$ from $t=1$ to $t=t_i$. The conditional probability distribution can be expressed in its reparameterized form as:
 \begin{equation}
 	X_{t_{i-1}}=X_{t_i}-\epsilon_{\theta}\Delta_t+\sigma_t Z, 
 	\label{stage1}
 \end{equation}
 where $\sigma_t=\sqrt{\left[kt_{i-1}(1-t_{i-1})\right]^2-\left[kt_{i-1}^{\alpha}t_i^{1-\alpha}(1-t_i)\right]^2}$, 
 	 $\Delta_t=t_i^{\alpha}-t_{i-1}^{\alpha}$, $t_{i-1} \geq \frac{1-\alpha}{2-\alpha}$.
 It is important to note that this formulation is only valid within the interval $[\frac{1-\alpha}{2-\alpha},1]$, which constitutes the fundamental rationale for our two-stage sampling approach. 
 Derivation of Eq.~(\ref{stage1}) can be found in Sec.~\ref{proof:sp}.

 \begin{table*}
  \caption{Comparison with PMRF in numerous image restoration tasks.}
  \label{tab:pmrf}
  \centering
  \footnotesize
  \renewcommand{\arraystretch}{0.85}
  \setlength{\tabcolsep}{2pt} 
  \begin{tabular}{@{}l|cccccc|cccccc|cccccc@{}}
    \toprule
\multicolumn{1}{c|}{\textbf{Method}} &   \multicolumn{6}{c|}{\textbf{JPEG Restoration}} & \multicolumn{6}{c|}{\textbf{4$\times$ Super-Resolution}} & \multicolumn{6}{c}{\textbf{Deblurring}} \\
    \midrule
    &QF&NFE&FID$\downarrow$&PSNR$\uparrow$& SSIM$\uparrow$&LPIPS$\downarrow$&Filter&NFE&FID$\downarrow$&PSNR$\uparrow$&SSIM$\uparrow$&LPIPS$\downarrow$&Kernel&NFE&FID$\downarrow$&PSNR$\uparrow$&SSIM$\uparrow$&LPIPS$\downarrow$\\
\cmidrule(lr){2-19}
PMRF&\multirow{4}{*}{5}&\multirow{2}{*}{10}&7.6&24.62&0.69&0.31&\multirow{4}{*}{Pool}&\multirow{2}{*}{10}&5.1&25.24&0.73&0.21&\multirow{4}{*}{Uniform}&\multirow{2}{*}{10}&6.6&27.84&0.82&0.20 \\ 
NADB&                  &                   &6.9&24.45&0.69&0.30&                     &                   &5.3&24.75&0.71&0.23&                        &                   &4.8&27.70&0.81&0.18  \\
 \cmidrule(r){1-1} \cmidrule(lr){3-7} \cmidrule(lr){9-13} \cmidrule(lr){15-19} 
PMRF&                  &\multirow{2}{*}{100}&7.8&23.69&0.64&0.32&                    &\multirow{2}{*}{100}&4.5&24.03&0.69&0.21&                       &\multirow{2}{*}{100}&7.2&27.53&0.81&0.21 \\ 
NADB&                  &                    &4.3&23.63&0.65&0.30&                    &                    &3.7&23.63&0.67&0.22&                       &                    &3.4&27.39&0.80&0.17  \\

    \bottomrule
  \end{tabular}
\end{table*}
 \subsection{The Second Stage}
 The direct application of Eq.~(\ref{stage1}) over the interval $[0, \frac{1-\alpha}{2-\alpha}]$ is prohibited due to the non-negativity requirement of its standard deviation term $\sigma_t$. To overcome this limitation, we introduce a hyperparameter $w$ into the noise term of $X_{t_{i-1}}\sim p(X_{t_{i-1}}\mid X^{\epsilon}_0, \hat{X}_0, X_{t_i})$ , yielding the modified expression:
 \begin{equation}
 	X_{t_{i-1}}=a\epsilon_{\theta}-bX_{t_i}+c\hat{X}_0+\sigma_{wt} Z,
 	\label{stage2}
 \end{equation}
 where
 \begin{equation}
     \begin{aligned}
         a&=t_i^{\alpha}\left(w+t_{i-1}^{\alpha}-1-wt_i^{\alpha}\right), \\
         b&=t_{i-1}^{\alpha}-1-wt_i^{\alpha}, \\
         c&=t_{i-1}^{\alpha}-wt_i^{\alpha}, \\
         \sigma_{wt}&=\sqrt{\left[kt_{i-1}(1-t_{i-1})\right]^2-\left[kwt_i(1-t_i)\right]^2}.
     \end{aligned}
 \end{equation}
 It should be noted that Eq.~(\ref{stage2}) remains valid over the entire interval $[0,1]$. This allows the value of $t_i$ to be flexibly chosen (as long as $t_i \geq \frac{1-\alpha}{2-\alpha}$) in order to determine the specific subinterval over which Eq.~(\ref{stage2}) is applied. We set the partition point $d$ between the two sampling stages to $\frac{1-\alpha}{2-\alpha}$ by default.
 The hyperparameter $w$ controls the magnitude of $\sigma_t$, thereby determining the amount of stochastic noise introduced over the interval 
 $[0,t_i]$. When $\sigma_t=0$, the sampling process within this interval reduces to a deterministic trajectory. In our experiments, we set $w=\frac{t_{i-1}}{t_i}$ to evaluate the impact of varying noise levels on the final output.
 Derivation of Eq.~(\ref{stage2}) can be found in Sec.~\ref{proof:sp}.
 
\subsection{Derivation of Sampling Procedure}
\label{proof:sp}
\begin{proof}
Our derivation leverages the reparameterization technique. For $s < t$, we define the process as:
\begin{align}
X_s &= X_t - \frac{X_t - X^{\epsilon}_0}{t^{\alpha}} \Delta_t + \sigma_t Z \\
&= \frac{\Delta_t}{t^{\alpha}} X^{\epsilon}_0 + \frac{t^{\alpha} - \Delta_t}{t^{\alpha}} X_t + \sigma_t Z.
\end{align}
According to the calculation rules of reparameterization, the coefficients of Gaussian noise in the left and right equations are equal, and we obtain the following equations:
\begin{align}
\frac{\Delta_t}{t^{\alpha}} &= 1 - t^{\alpha}, \\
\left[k s (1-s)\right]^2 &= \left[ \frac{t^{\alpha} - \Delta_t}{t^{\alpha}} \cdot k t (1-t) \right]^2 + \sigma_t^2.
\end{align}
Solving this system yields $\Delta_t = t^{\alpha} - s^{\alpha}$ and
$\sigma_t = \sqrt{ \left[k s (1-s)\right]^2 - \left[k s^{\alpha} t^{1-\alpha} (1-t)\right]^2 }$.

The non-negativity constraint $\sigma_t \geq 0$ implies $s \geq \frac{1-\alpha}{2-\alpha}$, and thus $d \geq \frac{1-\alpha}{2-\alpha}$.

For the second stage from $t = d$ to $t = 0$, we introduce a hyperparameter $w$ and define $\sigma_{wt} = \sqrt{ \left[k s (1-s)\right]^2 - \left[w k t (1-t)\right]^2 }$. Using reparameterization:
\begin{align}
X_s &= (1 - s^{\alpha}) X^{\epsilon}_0 + s^{\alpha} \hat{X}_0 + k s (1-s) Z \\
&= (1 - s^{\alpha}) X^{\epsilon}_0 + s^{\alpha} \hat{X}_0 + \sigma_{wt} Z + w k t (1-t) Z,
\end{align}
and substituting the identity $w k t (1-t) Z = w [X_t - (1 - t^{\alpha}) X^{\epsilon}_0 - t^{\alpha} \hat{X}_0]$, we rearrange to express $X_s$ in terms of $X_t$, $X_1$ and $Z$:
\begin{align}
X_s &= \left[w t^{\alpha} - s^{\alpha} + 1\right] X_t + \left(s^{\alpha} - w t^{\alpha}\right) \hat{X}_0 + \sigma_{wt} Z \\
&\quad + t^{\alpha} \left[s^{\alpha} - 1 + w (1 - t^{\alpha})\right] \cdot \frac{X_t - X^{\epsilon}_0}{t^{\alpha}}.
\end{align}
This completes the derivation.
\end{proof}

\section{Proof of Theorem~\ref{thm:mean}}

\begin{theorem} \label{thm:mean_proof}
Let $ \rho_0$ be the target distribution, $\rho_1$ the degraded distribution, and $\hat{\rho}_0$ the output distribution of the mean network $M_\phi$ optimized via Eq.~(\ref{eq:mean_loss}). The mean network reduces the distributional gap in the Wasserstein-2 sense:
\begin{equation}
W_2(\rho_0,\hat{\rho}_0) \leq W_2(\rho_0,\rho_1).
\end{equation}
\end{theorem}
\begin{proof}
Given $(X_0 , X_1) \sim (\rho_0, \rho_1)$ and  $\hat{X}_0=M_{\phi}(X_1)=\mathbb{E}(X_0\mid X_1)$, there exits:
\begin{equation}
\label{eq:square}
 \mathbb{E} \left[ \| X_0 - \hat{X_0}\|^2 \right]\leq \mathbb{E} \left[ \| X_0 - X_1\|^2 \right].
\end{equation}
Eq.~(\ref{eq:square}) can be verified by expanding the square on the left-hand side. Using the definition of Wasserstein-2 distance and Eq.~(\ref{eq:square}), we obtain:
\begin{align}
    W_2^2\left(\rho_0, \rho_1\right) &=  \inf_{\gamma \in \Pi(\rho_0, \rho_1)} \mathbb{E}_{(X_0, X_1) \sim \gamma} \left[ \| X_0 - X_1 \|^2 \right]  \\
     &=  \mathbb{E}_{(X_0, X_1) \sim \gamma^*} \left[ \| X_0 - X_1 \|^2 \right] \\
     &\geq \mathbb{E}_{(X_0, X_1) \sim \gamma^*} \left[ \| X_0 - \mathbb{E}(X_0\mid X_1)\|^2 \right]\\
     &= \mathbb{E}_{(X_0, \hat{X}_0) \sim \pi^*} \left[ \| X_0 - \hat{X}_0\|^2 \right]\\
     &\geq \inf_{\pi \in \Pi(\rho_0, \hat{\rho}_0)} \mathbb{E}_{(X_0, \hat{X}_0) \sim \pi} \left[ \| X_0 - \hat{X}_0 \|^2 \right] \\
     &=W_2^2\left(\rho_0, \hat{\rho}_0\right)
\end{align}
\end{proof}

\section{Hyperparameters in Eq.~\ref{eq:madbm_x}}
This section outlines the selection of hyperparameters $\alpha$ and $k$ in Eq.~(\ref{eq:madbm_x}). The parameter $k$ controls the overall scale of the stochastic term. We calibrate $k$ so that the maximum amplitude of the stochastic term in Eq.~(\ref{eq:madbm_x}) closely matches that of I2SB~\cite{liu20232} in Eq.~(\ref{eq:i2sbxt}). This alignment follows the empirical setup of prior work, as neither our method nor I2SB~\cite{liu20232} systematically investigates the effect of the stochastic term's magnitude on generation quality. A comprehensive analysis of this effect will be explored in future work.

The value of $\alpha$ in the mean term is theoretically constrained to the range $(0, 1)$, based on noise alignment considerations. We empirically determine its optimal value through systematic experiments, with quantitative results presented in Tab.~\ref{tab:alpha}. Considering the performance across all metrics, we select $\alpha=0.4$.

\section{Supplement the Experimental and Qualitative Results}
\subsection{Compare with PMRF}
We conducted a systematic comparison between the NADB and the Flow Matching-based PMRF~\cite{ohayon2024posterior} method  under three different types of noise, using identical network architectures and strictly adhering to the original hyperparameters.  As shown in Tab.~\ref{tab:pmrf}, the PMRF~\cite{ohayon2024posterior}  method exhibits instability at higher NFE values. This result underscores the advantage of the diffusion bridge approach and highlights the importance of our research into the endpoint fitting problem.

\begin{table}[htbp]
  \caption{Hyperparameter tuning using JPEG restoration tasks, QF=5.}
  \label{tab:alpha}
  \centering
  \begin{tabular}{@{}lccccc@{}}
    \toprule
    \multicolumn{6}{c}{\textbf{JPEG Restoration}} \\
    \midrule
    NFE&$\alpha$ & FID$\downarrow$&PSNR$\uparrow$& SSIM$\uparrow$&LPIPS$\downarrow$ \\
    \midrule
    \multirow{3}{*}{10} 
      & 0.3 & 8.2&24.60&0.70&0.30 \\ 
      & 0.4 & 6.9&24.45&0.69&0.30 \\
      & 0.5 & 7.1&24.44&0.69&0.30 \\
    \midrule
    \multirow{3}{*}{100}
      & 0.3 & 4.5&23.76&0.66&0.30 \\
      & 0.4 & 4.3&23.63&0.65&0.30 \\
      & 0.5 & 4.4&23.56&0.64&0.31 \\
    \bottomrule
  \end{tabular}
\end{table}

\subsection{Additional  Results}
Tab.~\ref{tab:uni-dreslsr} provides additional image restoration, image translation, and super-resolution  comparison experiments.
Fig.~\ref{fig:eightjpeg} to~\ref{fig:eightsr4x} provide additional qualitative results on each restoration tasks,  Fig.~\ref{fig:16e2h},~\ref{fig:16e2s} provide additional qualitative results on each translation tasks, and Fig.~\ref{fig:gauss_pool_duibi} provides additional  examples comparing between NADB and I2SB w.r.t. various NFE sampling.

\begin{table}[t]
  \caption{Comparisons with additional baselines.}
  \label{tab:uni-dreslsr}
  \centering
  \scriptsize

  \begin{minipage}{\linewidth}
  \centering
  \setlength{\tabcolsep}{2pt}
  \begin{tabular*}{\linewidth}{@{\extracolsep{\fill}}lccccc|lcccc}
    \toprule
    \multicolumn{6}{c|}{\textbf{(a) Deblurring, kernel=uniform}} 
    & \multicolumn{4}{c}{\textbf{(b) Edges2Handbags (256$\times$256)}} \\
    \midrule
    Method  & NFE & FID$\downarrow$ & PSNR$\uparrow$ & SSIM$\uparrow$ & LPIPS$\downarrow$ 
    &Method  & PSNR$\uparrow$ & SSIM$\uparrow$ & LPIPS$\downarrow$\\
    \midrule
    I3SB & \multirow{2}{*}{10}
         & 6.5 & 24.17 & 0.65 & 0.32 
         &NADB&19.25&0.740&0.231\\
    NADB &
         & 4.8 & 27.70 & 0.81 & 0.18 
         &RDBM&19.26&0.738&0.224\\
    \cmidrule{1-6}
    I3SB & \multirow{2}{*}{100}
         & 4.9 & 22.70 & 0.58 & 0.33 
         & GOUB& 16.58 & 0.700 & 0.288 \\
    NADB &
         & 3.4 & 27.39 & 0.80 & 0.17 
         & RDDM& 14.66 & 0.645 & 0.256 \\
    \bottomrule
  \end{tabular*}
  \end{minipage}

  \begin{minipage}{\linewidth}
  \centering
  \begin{tabular*}{\linewidth}{@{\extracolsep{\fill}}lccccc@{}}
     \toprule
      \multicolumn{6}{c}{\textbf{(c) Comparison on DrealSR dataset, NFE=40}} \\
      \midrule
     Method & LPIPS$\downarrow$ & MUSIQ$\uparrow$ & MANIQA$\uparrow$ & ClipIQA$\uparrow$ & LIQE$\uparrow$ \\
    \midrule
    NADB   & 0.315 & 71.480 & 0.483 & 0.576 & 3.610 \\
    DiT4SR & 0.365 & 64.950 & 0.627 & 0.548 & 3.964 \\
    \bottomrule
  \end{tabular*}
  \end{minipage}

\end{table}

\begin{figure*}[p]
    \centering
    \includegraphics[width=0.95\textwidth]{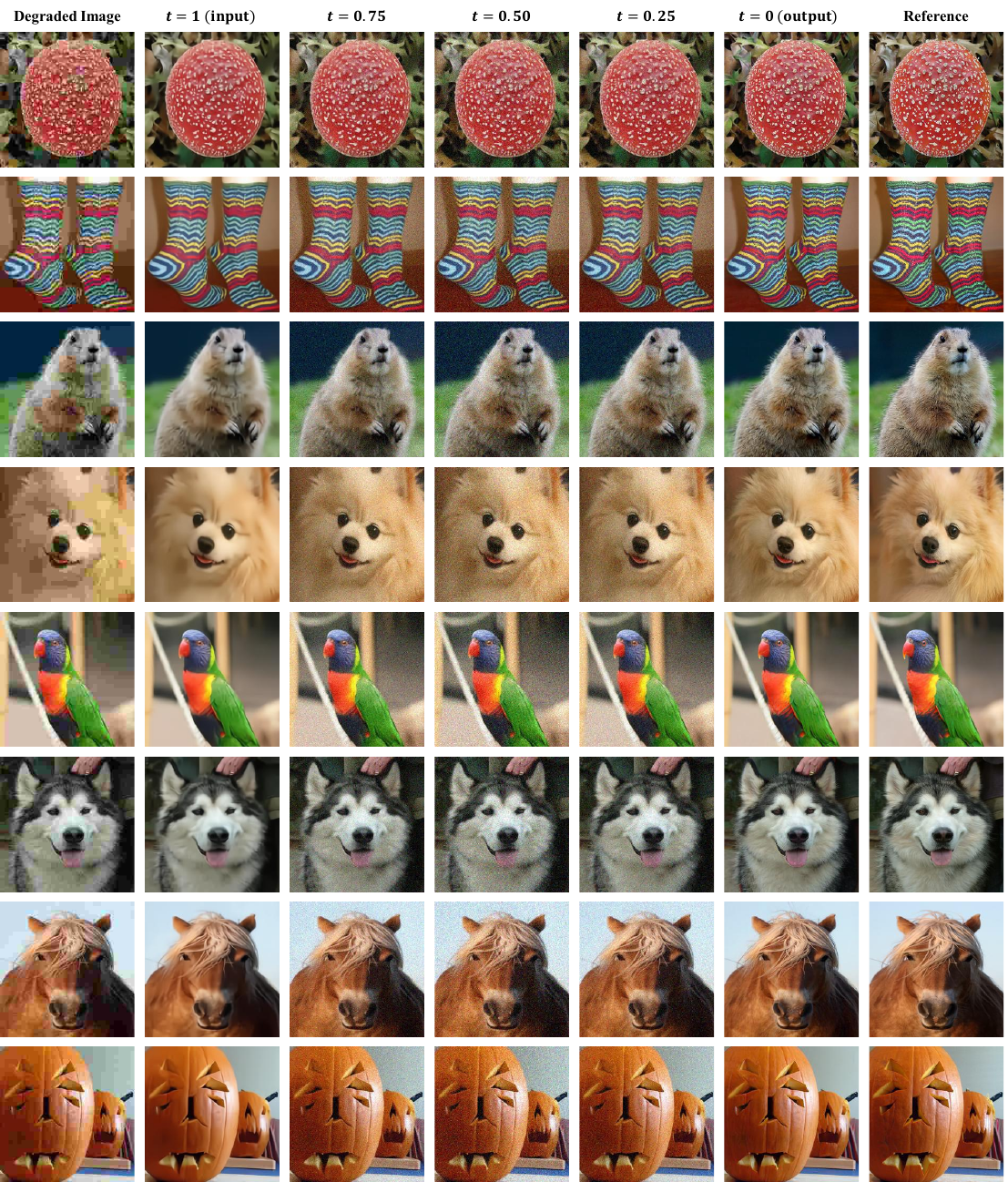}
    \caption{Generative processes of NADB on JPEG restoration tasks. \textbf{Top 4 rows}: QF=5. \textbf{Bottom 3 rows}: QF=10.}
    \label{fig:eightjpeg}
\end{figure*}
\begin{figure*}[p]
    \centering
    \includegraphics[width=0.95\textwidth]{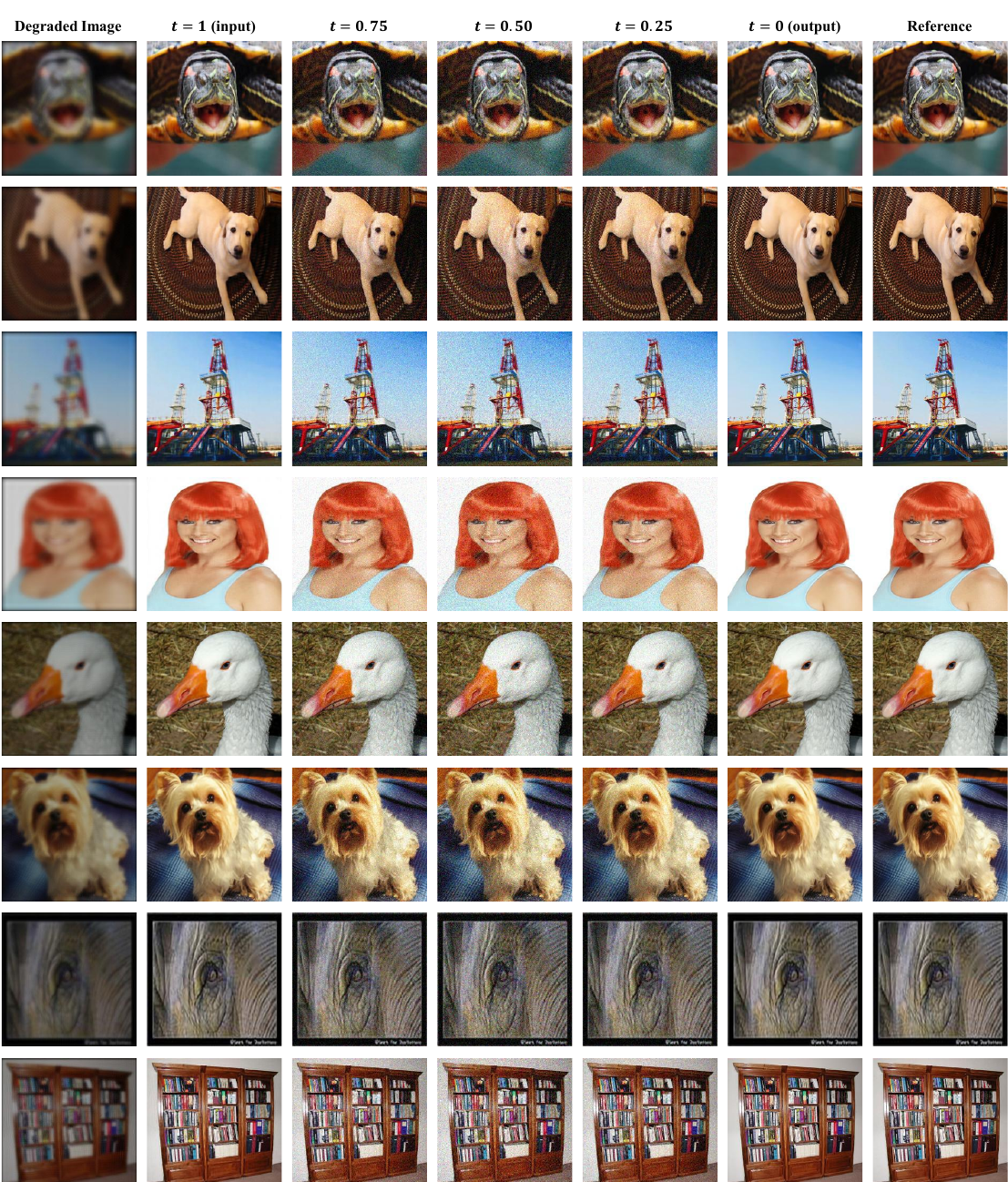}
    \caption{Generative processes of NADB on deblurring tasks. \textbf{Top 4 rows}: Uniform kernel. \textbf{Bottom 4 rows}: Gaussian kernel.}
    \label{fig:eightblur}
\end{figure*}
\begin{figure*}[p]
    \centering
    \includegraphics[width=0.95\textwidth]{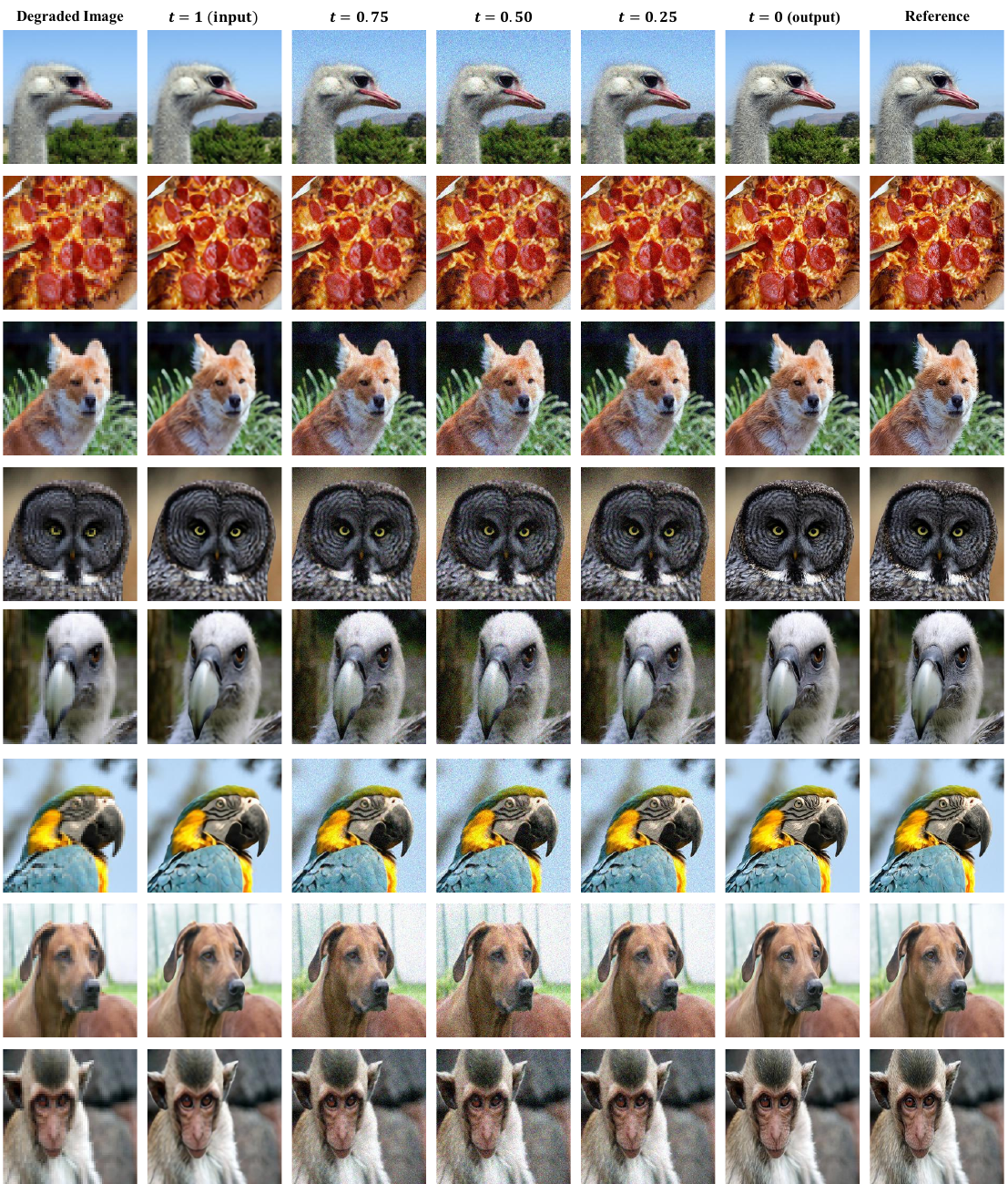}
    \caption{Generative processes of NADB on 4 $\times$ super-resolution tasks. \textbf{Top 4 rows}: Pool kernel. \textbf{Bottom 4 rows}: Bicubic kernel.}
    \label{fig:eightsr4x}
\end{figure*}

\begin{figure*}[p]
    \centering
    \includegraphics[width=0.95\textwidth]{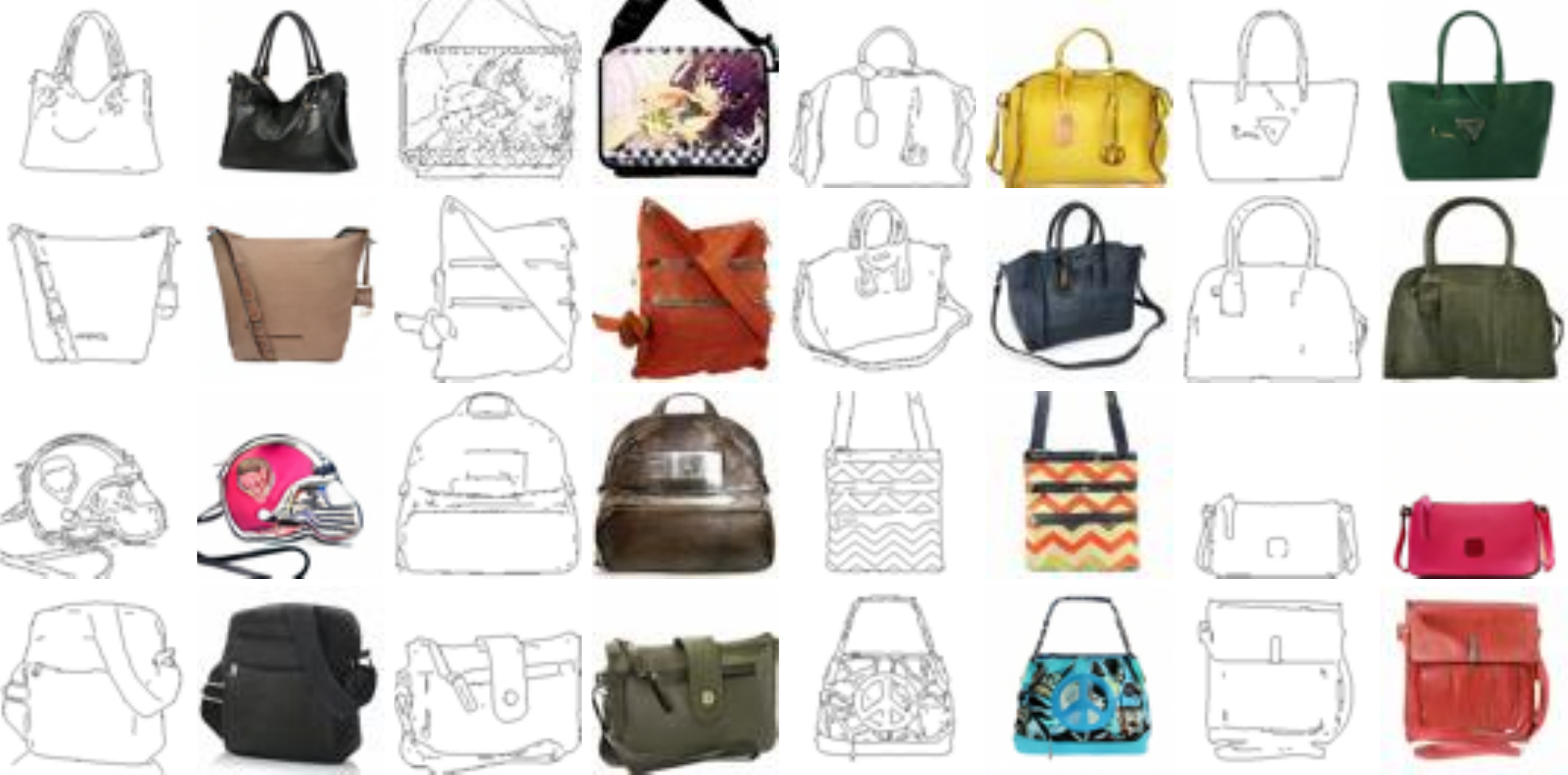}
    \caption{Generate results on edges2handbags dataset with resolution $64 \times 64$.}
    \label{fig:16e2h}
\end{figure*}
\begin{figure*}[p]
    \centering
    \includegraphics[width=0.95\textwidth]{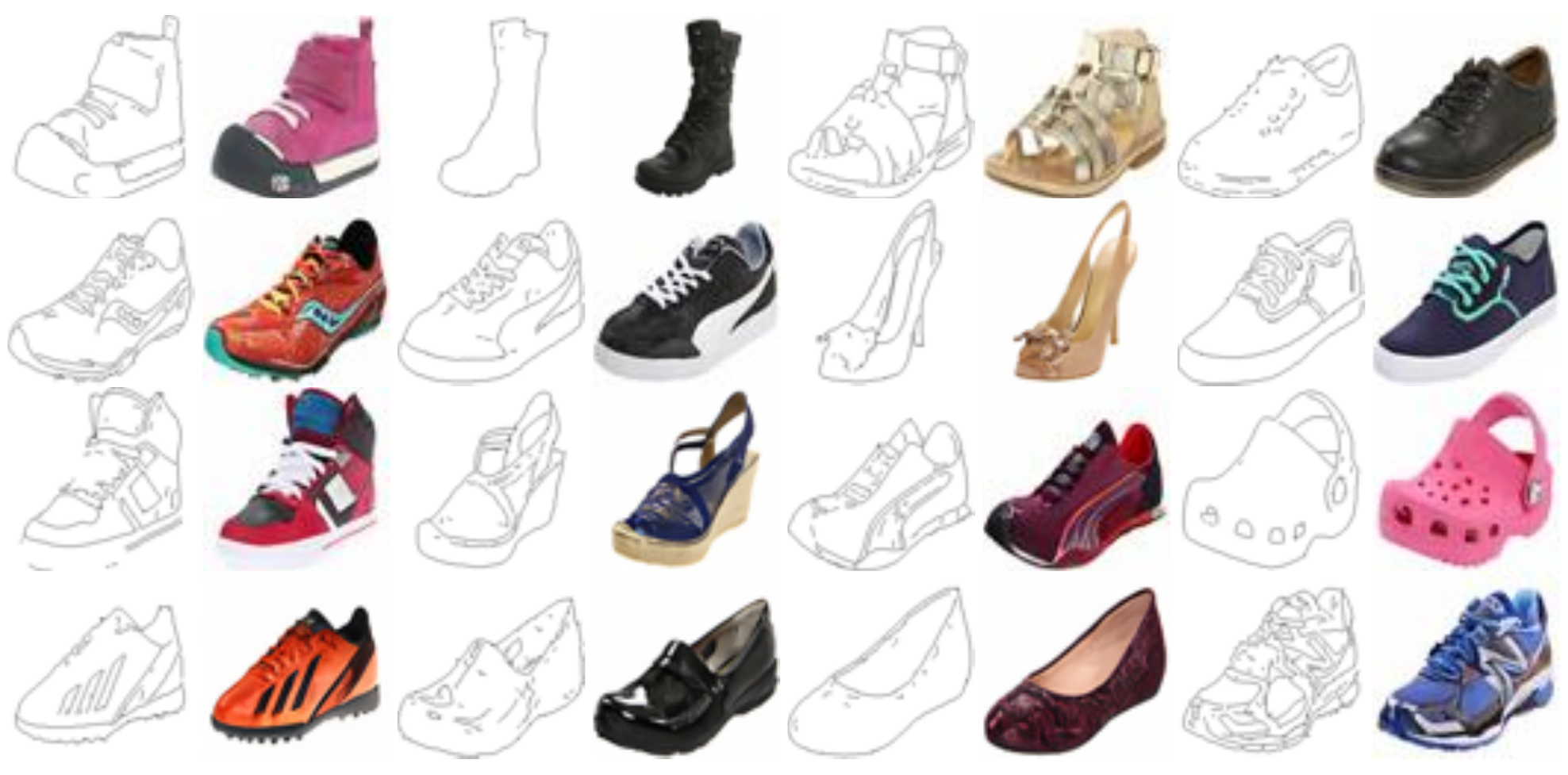}
    \caption{Generate results on edges2shoes dataset with resolution $64 \times 64$.}
    \label{fig:16e2s}
\end{figure*}

\begin{figure*}[p]
    \centering
    \includegraphics[width=0.89\textwidth]{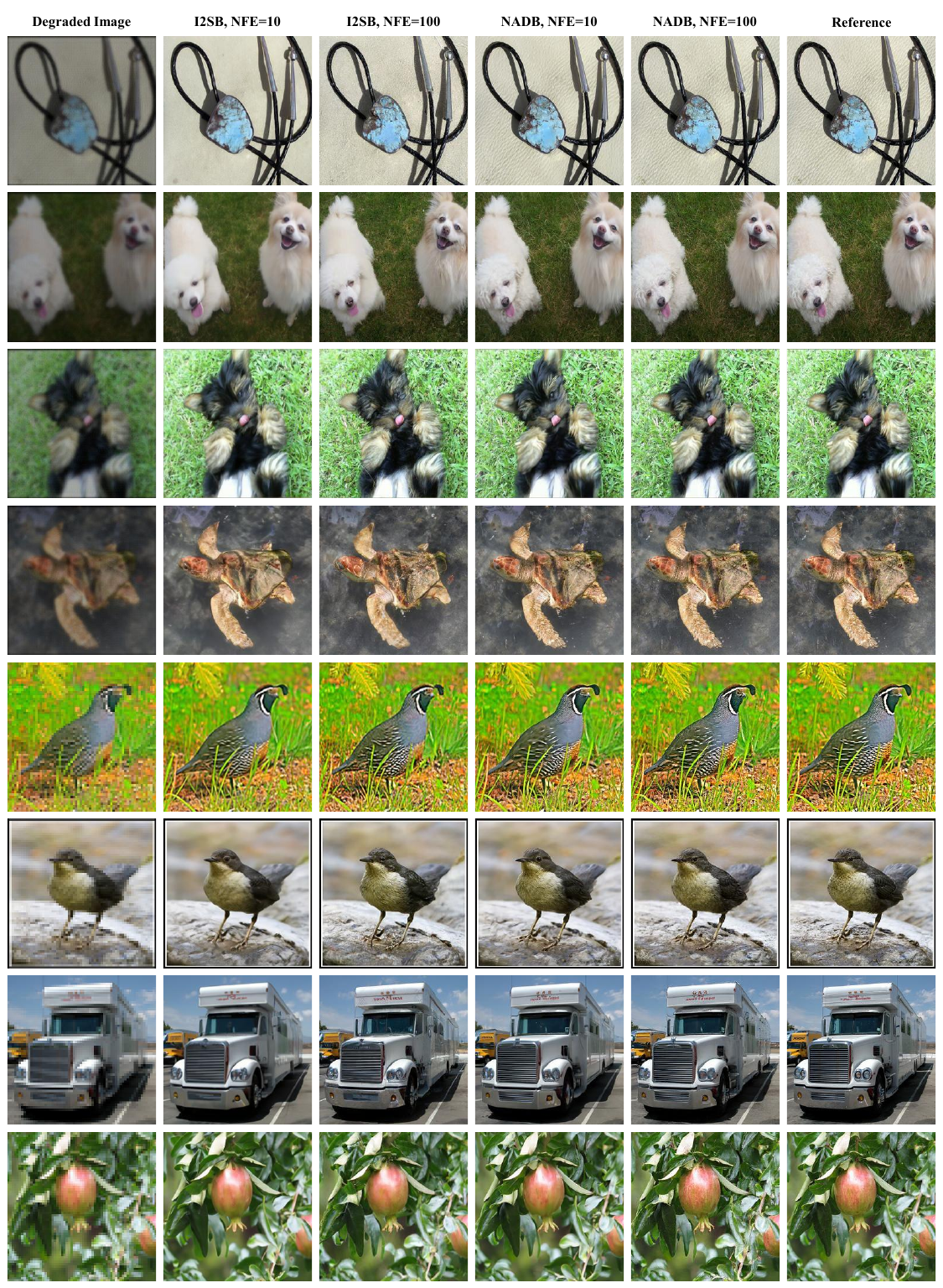}
    \caption{Additional qualitative comparison between NADB and I2SB. \textbf{Top 4 rows}: Deblurring tasks with gaussian kernel. \textbf{Bottom 4 rows}: 4 $\times$ super-resolution tasks with pool kernel.}
    \label{fig:gauss_pool_duibi}
\end{figure*}